\documentclass[letterpaper, 10 pt, conference]{ieeeconf}  

\IEEEoverridecommandlockouts                              

\usepackage[pass]{geometry}
\usepackage{fancyhdr,todo}
\usepackage{amsmath}
\usepackage{amsfonts,amssymb,float,graphicx}

\usepackage{amsthm}

\usepackage{subfigure}
\usepackage[usenames]{color}
\usepackage{pifont}

\usepackage{multirow}
\usepackage{hhline}

\usepackage{listings}
\usepackage{color} 
\definecolor{mygreen}{RGB}{28,172,0} 
\definecolor{mylilas}{RGB}{170,55,241}

\usepackage{algpseudocode,algorithm,algorithmicx}

\theoremstyle{plain}
\newtheorem{theorem}{Theorem}[section] 

\theoremstyle{definition}
\newtheorem{assumption}[theorem]{Assumption}
\newtheorem{lemma}[theorem]{Lemma}

\usepackage[font=footnotesize,labelfont=bf]{caption}

\usepackage{cite}

\title{\LARGE \bf
Distributed off-Policy Actor-Critic Reinforcement Learning \\
with Policy Consensus
}

\author{Yan Zhang and Michael M. Zavlanos
\thanks{Yan Zhang and Michael M. Zavlanos are with the Department of Mechanical Engineering and Materials Science, Duke University, Durham, NC 27708, USA. {\tt\small \{yan.zhang2,michael.zavlanos\}@duke.edu}}%
}

\makeatletter
\newcommand\fs@spaceruled{\def\@fs@cfont{\bfseries}\let\@fs@capt\floatc@ruled
	\def\@fs@pre{\vspace{1.5\baselineskip}\hrule height.8pt depth0pt \kern2pt}%
	\def\@fs@post{\kern2pt\hrule\relax}%
	\def\@fs@mid{\kern2pt\hrule\kern2pt}%
	\let\@fs@iftopcapt\iftrue}
\makeatother

\begin{document}

\maketitle
\thispagestyle{empty}
\pagestyle{empty}

\begin{abstract}
In this paper, we propose a distributed off-policy actor critic method to solve multi-agent reinforcement learning problems. Specifically, we assume that all agents keep local estimates  of the global optimal policy parameter and update their local value function estimates independently. Then, we introduce an additional consensus step to let all the agents asymptotically achieve agreement on the global optimal policy function. The convergence analysis of the proposed algorithm is provided and the effectiveness of the proposed algorithm is validated using a distributed resource allocation example. Compared to relevant distributed actor critic methods, here the agents do not share information about their local tasks, but instead they coordinate to estimate the global policy function.
\end{abstract}

\section{INTRODUCTION}
\label{sec:intro}
Reinforcement learning (RL) algorithms have been widely used to solve decision making and control problems in unknown and stochastic environments, \cite{mnih2013playing,mnih2015human}. Existing RL algorithms fall in two main categories, tabular-based methods and methods that use function approximation. Tabular-based methods are generally easier to analyze \cite{bertsekas95dynamic}, however, they require the state and action spaces to be discrete and finite. On the other hand, using function approximation, such as Neural Networks \cite{mnih2015human}, allows to solve RL problems in continuous state and action spaces. The goal of these methods is to estimate the value function or policy function over the whole state-action space with a finite number of function parameters. Then, the learning problem can be reduced to finding the optimal function parameters in finite dimensions. However, methods that rely on function approximation can be sensitive to approximation errors and can diverge in some cases \cite{baird1995residual}. Understanding convergence of RL algorithms with function approximation is an active area of research that is also the focus of this work.

Existing RL algorithms can also be classified as value-based methods \cite{sutton2008convergent,maei2010toward} or policy gradient methods \cite{sutton2000policy,silver2014deterministic}. Value-based methods parameterize the state-value function $V(s)$ or state-action value function $Q(s,a)$ and learn these function parameters during the learning process. In these methods, in order to obtain the control signal, a maximization problem $\max_a Q(s,a)$ needs to be solved at each time step of the execution phase, which is impractical. Instead, policy gradient methods parameterize and directly learn the policy function using stochastic gradient descent, \cite{sutton2000policy,silver2014deterministic}. It is well-known that in these methods the estimate of the policy gradient typically has large variance and one popular way to reduce this variance is the Actor-Critic (AC) method, \cite{konda2000actor}. Essentially, the learning agent keeps a policy function estimator called Actor and a value function estimator called Critic. The Critic estimates the value function under the current policy and the Actor uses the feedback from the Critic to improve the policy function parameters. 

In this paper, we are interested in distributed Actor-Critic methods. Specifically, we consider networks of agents that have their own tasks, states and actions, and assume that their states depend not only on their own actions but also on the states and actions of the other agents in the network. The goal is to let the agents in the network collaborate to learn a global optimal policy that maximizes the aggregate accumulated rewards over the network. The challenge in applying the AC method in this scenario is that the Critic update needs the local reward information from all the agents. This usually requires a master node to serve as a central critic, \cite{foerster2017counterfactual,srinivasan2018actor}. When the network size is large, having a master node introduces significant communication overhead and may also cause privacy issues. The works in \cite{macua2015distributed, stankovic2016multi,yuan2017variance,wai2018multi,lee2018primal} employ distributed optimization methods to evaluate a given fixed policy in multi-agents systems. However, these methods do not improve the policy parameters. A different formulation is proposed in \cite{pennesi2010distributed} that develops a distributed Actor Critic method for teams of homogeneous agents that learn the same task independently and share their policy parameters with their neighbors. Instead, here we assume that the agents can have different tasks, different state and action spaces, and their behavior can affect each other. The works in \cite{lowe2017multi} study the multi-agent actor critic method in mixed competitive-cooperative environment but with no convergence analysis.

Perhaps the most relevant work to the method proposed here is \cite{zhang2018fullyICML,zhang2018networked}. The key idea in these works is to let each agent keep a local estimate of the global value function and introduce an additional consensus step on these local estimates to make the local agents asymptotically aware of the global value function. Compared to \cite{zhang2018fullyICML,zhang2018networked}, here every agent keeps its own local value function estimate associated with their own task and never shares this with its neighbors. Therefore, information about the local tasks is not revealed to other agents, as in \cite{zhang2018fullyICML,zhang2018networked}. Instead, the agents keep local estimates of the global policy function and a consensus step is introduced so that all agents agree on the optimal policy. 

The rest of the paper is organized as follows. In Section~\ref{sec:prelim}, we introduce the distributed reinforcment learning problem under consideration as well as some preliminary results. In Section~\ref{sec:pf}, we formulate the decentralized off-policy reinforcement learning problem and formally present our proposed distributed Actor-Critic algorithm. In Section~\ref{sec:convergence}, we analyze the convergence of the proposed algorithm. In Section~\ref{sec:simulation}, we present a numerical example to validate our analysis. In Section~\ref{sec:conclusion}, we conclude the paper.

\section{Preliminaries}
\label{sec:prelim}
\subsection{The Reinforcement Learning Problem}
Consider network of $N$ agents. We define the state of the system $s = [s_1, s_2, \dots, s_N] \in \mathcal{S}$, where $s_i \in \mathcal{S}_i$ denotes the state of agent $i$. Moreover, we define the action of the system $a = [a_1, a_2, \dots, a_N] \in \mathcal{A}$, where $a_i \in \mathcal{A}_i$ denotes the action of agent $i$. The state space $\mathcal{S}_i$ and action space $\mathcal{A}_i$ are continuous. We denote the state transition function by $s(t+1) = T_t(s(t), a(t), \omega(t))$, where $\omega(t)$ represents the noise in the state dynamics at time $t$. We also assume that the transition process is stationary, that is, the function $T_t$ and the noise $\omega(t)$ generate a time-invariant distribution $P(s(t+1)|s(t), a(t))$. We assume that the global state and action can be observed by all the agents. This is a common assumption in current reinforcement learning problems, \cite{lowe2017multi,foerster2017counterfactual,zhang18continuous,zhang2018fullyICML}.

Let $r_i(s(t), a(t))$ denote the local reward received by agent $i$ at time $t$ as a result of taking action $a(t)$ at state $s(t)$. Define also the the global deterministic policy function, $\pi(s): \mathcal{S} \rightarrow \mathcal{A}$. Moreover, denote by $V^\pi(s) = \mathbb{E}_{\rho^\pi}[\sum_{t=0}^\infty \gamma^t \sum_{i=1}^N r_i(s(t), a(t)) | s(0) = s, \pi]$ the state value function and by $Q^\pi(s,a) = \mathbb{E}_{\rho^\pi} [\sum_{t=0}^\infty \gamma^t \sum_{i=1}^N r_i(s(t), a(t)) | s(0) = s, a(0) = a, \pi]$ the state-action value function. The goal is to maximize the infinite-time discounted-reward value function
\begin{equation}
	\label{eq:pf}
	J(\pi) = \mathbb{E}_{\rho^\pi}[ V^\pi(s(0))],
\end{equation}
over all policies $\pi\in \Pi$ where $\Pi$ is the policy function space, $\gamma \in (0,1)$ is the discounted factor. Assuming the transition probability $P(s(t+1)|s(t), a(t))$ and the reward $r_i(s(t), a(t))$ are known, the existence of the optimal stationary policy function $\pi^\ast(s)$ that maximized the value function \eqref{eq:pf} is shown in \cite{bertsekas95dynamic} and this policy can be found using planning methods, e.g., policy iteration. However, if the probability and reward functions are unknown, reinforcement learning methods need to be applied to find the optimal policy function $\pi^\ast(s)$. In this paper, we are interested in actor critic methods to find the optimal policy $\pi^*(s)$.

\subsection{Actor Critic Method}
\label{subsec:ac}
Parameterizing the policy function as $\pi_\theta(s) = \phi_\pi(s)^T \theta$, where $\theta \in \mathbb{R}^{n_\theta}$ is the policy parameter and $\phi_\pi(s): \mathcal{S} \rightarrow \mathbb{R}^{n_\theta}$ is a vector of $n_\theta$ policy feature functions, the problem of finding the optimal policy that maximizes the value function in \eqref{eq:pf} can be reduced to the following optimization problem in the parameterized function space $\Pi_\theta$,
\begin{equation}
\label{eq:pf_theta}
\max_{\theta} J(\theta).
\end{equation}
Problem~\eqref{eq:pf_theta} can be solved using stochastic gradient descent methods. Specifically, the gradient $\nabla_\theta J$ is given in \cite{silver2014deterministic},
\begin{equation}
\label{eq:deter_policy_grad}
\nabla_\theta J(\pi) = \mathbb{E}_{\rho^\pi}[\nabla_\theta \pi(s) \nabla_a Q^\pi(s, a)|_{a = \pi_\theta(s)}].
\end{equation}
Since function $Q^\pi(s,a)$ under policy $\pi_\theta$ is unknown, policy evaluation algorithms are needed to approximate $Q^\pi(s,a)$ and furthermore the gradient $\nabla_\theta J(\pi)$. To do so, the function $Q^\pi(s,a)$ can be parameterized as $Q^\pi(s,a) = \phi_{Q}(s,a)^T w$, $w \in \mathbb{R}^{n_{w}}$, where $\phi_Q(s,a)$ is a vector of $n_w$  feature functions. Then, to solve problem~\eqref{eq:pf_theta} we can use the following actor-critic algorithm consisting of the two time scale updates
\begin{equation}
	\label{eq:central_ac}
	\begin{split}
	& w(t+1) = w(t+1) + \alpha_w(t) (h(w(t), \theta(t)) + M(t+1)), \\
	& \theta(t+1) = \theta(t) + \alpha_\theta(t) (f(w(t), \theta(t)) + N(t+1)),
	\end{split}
\end{equation}
where $h(w(t), \theta(t))$ represents the update formulas of a policy evaluation algorithm, e.g.,\cite{sutton2008convergent,maei2010toward}, $f(w(t), \theta(t))$ represents the policy gradient update in \eqref{eq:deter_policy_grad}, and $M(t+1)$ and $N(t+1)$ are noises coming from sampling during the learning process. The first update in \eqref{eq:central_ac} estimates the value function given the current policy, and is named Critic update. The second update in \eqref{eq:central_ac} improves the current policy, and is named Actor update. Convergence of this Actor Critic scheme typically depends on analyzing the two time scale updates \cite{borkar2009stochastic}, which we explain in detail in Section~\ref{sec:convergence}.

\section{Problem formulation}
\label{sec:pf}

Since the estimation of the global state-action value function $Q^\pi(s,a)$ in \eqref{eq:central_ac} requires the global reward function $r(t) = \sum_{i=1}^N r_i(s(t), a(t))$, the actor-critic method in~\eqref{eq:central_ac} is centralized. That is, a master node is required to collect the local rewards from all the agents and execute the update~\eqref{eq:central_ac}. To design a decentralized algorithm, we first decompose the value and action value functions using linearity of the expectation as
\begin{equation*}
\label{eq:Ji}
V^\pi(s) = \sum_{i = 1}^N V_i^\pi(s) \text{ and } Q^\pi(s,a) = \sum_{i=1}^{N} Q_i^\pi(s,a),
\end{equation*}
where $V_i^\pi(s)$ is the local value function and $Q_i^\pi(s,a)$ is the local action value function under policy $\pi$. Then, problem~\eqref{eq:pf} can be written in the following separable form
\begin{equation}
\label{eq:dec_rl}
\min_{\theta} \sum_{i = 1}^N J_i(\theta),
\end{equation}
where $J_i(\theta) = \mathbb{E}[ V_i^\pi(s(0)) | \rho(s(0)), \pi_\theta]$.
A common approach to solve problem~\eqref{eq:dec_rl} in a distributed way is by introducing local estimates $\theta_i \in \mathbb{R}^{n_\theta}$ and using consensus to estimate the global optimal policy parameter $\theta^\ast$:
\begin{equation}
\label{eq:consensus_optimization}
\theta_{i}(t+1) = \sum_{j \in \mathcal{N}_i} W_{ij} (\theta_j(t) + \alpha_\theta(t) \nabla_\theta J_j(\theta)|_{\theta = \theta_j(t)}).
\end{equation}
In the RL problem under consideration, the objective function $J_i(\theta)$ is usually nonlinear and the gradient $\nabla_\theta J_i$ is usually evaluated in a stochastic way. The convergence of \eqref{eq:consensus_optimization} in this case is studied in \cite{bianchi2013convergence}. The key idea in \cite{bianchi2013convergence} is to evaluate the local gradient $\nabla_\theta J_j(\theta)|_{\theta = \theta_j(t)}$ in an on-policy fashion, as in \cite{silver2014deterministic}, for which all agents need to behave under policy $\pi_{\theta_j(t)}$. This suggests that to execute the consensus update \eqref{eq:consensus_optimization} at time $t$, every agent $i$ needs to send its policy parameter $\theta_i(t)$ to all other agents that need to execute the policy $\pi_{\theta_i(t)}$ for multiple time steps so that agent $i$ can collect local rewards to estimate $\nabla_\theta J_i(\theta)|_{\theta = \theta_i(t)}$. This on-policy scheme is impractical even though its convergence analysis is simpler, as seen in \cite{bianchi2013convergence}. 

This motivates us to consider off-policy actor-critic methods as in \cite{silver2014deterministic,degris2012off}. The idea is to let all agents behave under a fixed policy, named behavorial policy $\beta(s)$, and optimize an approximate objective function,
\begin{equation}
\label{eq:approx_objective}
J_\beta(\theta) = \mathbb{E}_{\rho^\beta}[ V^\pi(s)],
\end{equation}
instead of the true cost $J(\theta)$ in \eqref{eq:pf}, where the expection in \eqref{eq:approx_objective} is taken over the stationary state distribution $\rho^\beta$ instead of $\rho^\pi$ as in \eqref{eq:pf}. Same as with $J(\theta)$, the approximate cost $J_{\beta}(\theta)$ can also be decomposed into local costs as
\begin{equation}
\label{eq:approx_dec_rl}
\min_{\theta} \sum_{i = 1}^N J_{i,\beta}(\theta).
\end{equation}
To achieve consensus on the local policy parameters $\theta_i$, we can apply a similar update as in \eqref{eq:consensus_optimization},
\begin{equation}
\label{eq:approx_consensus_opt}
\theta_{i}(t+1) = \sum_{j \in \mathcal{N}_i} W_{ij} (\theta_j(t) + \alpha_\theta(t) \nabla_\theta J_{j,\beta}(\theta)|_{\theta = \theta_j(t)}).	
\end{equation}
However, as discussed in \cite{silver2014deterministic}, the gradient $\nabla J_{j,\beta}(\theta)|_{\theta = \theta_j(t)}$ cannot be exactly estimated in this off-policy setting, therefore, it is replaced with an approximate gradient
\begin{equation}
\label{eq:approx_policy_gradient}
\hat{\nabla}J_{j,\beta}(\theta) = \mathbb{E}_{\rho^\beta}[\nabla_\theta \pi(s) \nabla_a Q_j^\pi(s, a)|_{a = \pi_\theta(s)}].
\end{equation}

Convergence of the Actor Critic method in \eqref{eq:approx_consensus_opt} using the off-policy gradient in \eqref{eq:approx_policy_gradient} is studied in \cite{degris2012off} for the centralized problem when no information is received from the neighbors. However, convergence of \eqref{eq:approx_consensus_opt} with off-policy gradient in \ref{eq:approx_policy_gradient} for decentralized problems is unknown, which is the focus of this paper.

In practice, we compute the gradient $\nabla_a Q_j^\pi(s,a)$ in \eqref{eq:approx_policy_gradient} using the parameterization $Q_i^\pi(s,a) \approx \sum_{p=1}^{n_w} \phi_{w_i}(s,a)^T w_i$. To ensure that this parameterization preserves the update \eqref{eq:approx_policy_gradient} that uses the true state-action value function $Q_i^\pi(s,a)$, as discussed in \cite{silver2014deterministic}, we make the following assumption:

\begin{assumption}
	\label{assum:function_compatibility}
	(Function Compatibility) All the local value functions $Q_i^\pi(s,a)$ are parameterized as $Q_i^\pi(s,a) = (a - \pi(s))^T(\nabla_\theta \pi(s))^T w_i$. That is, the feature function for $Q_i^\pi(s,a)$ satisfies that $\phi_{w_i}(s,a) = \nabla_\theta \pi(s) (a - \pi(s))$.
\end{assumption}
Given the Assumption~\ref{assum:function_compatibility} and the expression for the off-policy gradient~\eqref{eq:approx_policy_gradient}, the consensus update~\eqref{eq:approx_consensus_opt} of the local policy parameters $\theta_i$ becomes
\begin{equation}
\label{eq:update_theta}
\theta_i(t) = \sum_{j \in \mathcal{N}_i} W_{ij} (\theta_i(t-1) + \alpha_\theta \nabla \pi(s(t)) \nabla_\theta \pi(s(t))^T w_i(t).
\end{equation}

To conduct the policy parameter update in \eqref{eq:update_theta}, we have to compute $w_i(t)$, which is the value function parameter. To compute this parameter, generally speaking, any off-policy policy evaluation algorithm can be used at the local agents independently. Then, combining these policy evaluation algorithms with the approximate local gradient update in \eqref{eq:approx_policy_gradient}, a decentralized Actor Critic algorithm can be developed. In this paper, we employ the gradient temporal difference (GTD)  learning algorithm studied in \cite{sutton2008convergent,maei2010toward,silver2014deterministic} to estimate the local value function parameters $w_i$ in \eqref{eq:update_theta}, because this method is known to be stable in the off-policy setting.
The proposed distributed Actor Critic method is presented in Algorithm~\ref{alg:dac}.

\begin{algorithm}[t]
	\caption{Distributed Actor Critic with policy consensus}
	\label{alg:dac}
	\begin{algorithmic}[1]
		\Require{Initial value function parameters $\{w_i(0)\}$ and policy parameters $\{\theta_i(0)\}$. Step sizes $\alpha_w(0)$ and $\alpha_\theta(0)$. Set $t = 0$. Maximum time limit $T$. Discount factor $\gamma$. Fixed behavorial policy function $\beta(s)$. Agents' inital state $s(0)$.}
		\For{$0 \leq t \leq T$}
		\State{All agents take actions according to $\beta(s(t))$ and observe the actions $a(t)$, next states $s(t+1)$ and the local rewards $r_i(t)$.
			}
		\State{Every agent $i$ updates its local value function estimate $w_i$ based on the observed transition $[s(t), a(t), s(t+1)]$ and local reward $r_i(t)$ using an off-policy policy evaluation algorithm.
			}
		\State{Every agent $i$ updates its local policy parameter $\theta_i(t)$ according to equation \eqref{eq:update_theta}.
		}
	\EndFor
\end{algorithmic}
\end{algorithm}


\section{Convergence Analysis}
\label{sec:convergence}

In this section, we analyze the convergence of Algorithm~\ref{alg:dac} using the two-time scale technique in \cite{borkar2009stochastic}. The key idea is that the Critic updates at a faster rate than the Actor so that to analyze the convergence of the Critic update in line 3 in Algorithm~\ref{alg:dac}, we can assume that each local policy parameter $\theta_i$ is fixed. Then, each local Critic can independently estimate its own local value function and the convergence analysis of this local policy evaluation is the same as GTD in \cite{maei2010toward}. To analyze the convergence of the Actor, we can assume that every local Critic has already converged to the correct value function estimate. Compared to the analysis for the centralized off-policy Actor Critic method in \cite{degris2012off}, the challenge here is to analyze the decentralized off-policy Actor-Critic algorithm with a consensus step in line 4 of Algorithm~\ref{alg:dac}. To do so, we first introduce several assumptions that are common in the reinforcement learning literature.

\begin{assumption}
	\label{assum:aperiordic}
	We assume that the behavioral policy $\beta(s)$ is stationary, and the Markov chain that governs the state $s(t)$ under policy $\beta(s)$ is irreducible and aperiodic.
\end{assumption}
The above assumption ensures that when the agents behave under policy $\beta(s)$, the system states will reach the stationary state distribution $\rho^\beta$.

\begin{assumption}
	\label{assum:bounded_reward}
	We assume that for all $i$, the reward $|r_i(s, a)|$ is uniformly bounded for all $s$ and $a$.
\end{assumption}
This assumption ensures that the objective function in \eqref{eq:approx_dec_rl} is upper bounded when the discount factor satisfies $0 < \gamma < 1$, so that the problem~\eqref{eq:approx_dec_rl} is well defined.

\begin{assumption}
	\label{assum:stepsize}
	We assume that the stepsizes $\alpha_w(t)$ and $\alpha_\theta(t)$ are deterministic and satisfy that $\sum_{t=0}^{\infty} \alpha_w(t) \rightarrow \infty$, $\sum_{t=0}^{\infty} \alpha_\theta(t) \rightarrow \infty$, $\sum_{t=0}^{\infty} \alpha_w(t)^2 < \infty$ and $\sum_{t=0}^{\infty} \alpha_\theta(t)^2 < \infty$. Moreover, $\frac{\alpha_\theta(t)}{\alpha_w(t)} \rightarrow 0$.
\end{assumption}
This assumption is standard in the literature employing two-time scale analysis, \cite{borkar2009stochastic,zhang2018fullyICML,zhang2018networked}. Furthermore, let $W_t \in \mathbb{R}^{N \times N}$ be a random weight matrix of the communication graph at time $t$. Define the filtration $\mathcal{F}_t$ to be a $\sigma-$algebra $\sigma(\{\theta_i(0)\}, s(\tau), \{r_i(\tau)\}, W(\tau), \tau \leq t)$. Then, we have the following assumptions on $W_t$.
\begin{assumption}
	\label{assum:weight_matrix}
	We assume that $W_t$ satisfies the following conditions:
	(a) $W_t$ is row stochastic and $\mathbb{E}[W_t]$ is column stochastic for all $t>0$. That is, $W_t \mathbf{1} = \mathbf{1}$ and $\mathbf{1}^T\mathbb{E}[W_t] = \mathbf{1}^T$; (b) The spectral norm satisfies $\mathbb{E}[W_t^T(I - \frac{1}{N}\mathbf{1}\mathbf{1}^T)W_t] = \rho_W < 1$; (c) $W_t$ and $(s_t, r_t)$ are conditionally independent given the filtration $\mathcal{F}_{t-1}$.
\end{assumption}
The above assumptions are standard in the stochastic consensus optimization literature \cite{bianchi2013convergence}. They will be used to establish consensus on the policy parameters.

\begin{assumption}
	\label{assum:stability_phi}
	We assume the vector of feature functions $\phi_\pi(s)$ is uniformly bounded for all $s$.
\end{assumption}
This assumption is common and essential to show stability of reinforcment learning algorithms, see \cite{sutton2008convergent,maei2010toward,zhang2018networked,zhang2018fullyICML}.

\begin{assumption}
	\label{assum:stability_theta}
	We assume that through the whole history of the algorithm, $\{\theta_i(t)\}$ belongs to a compact set for all $i$ and $t$. We also assume that this compact set contains at least one local maximum of the problem \eqref{eq:approx_dec_rl}.
\end{assumption}
This assumption is necessary to show stability of the policy parameter updates. Moreover, boundedness of the policy parameters is commonly observed in practice when implementing RL algorithms as mentioned in \cite{degris2012off}. It is possible to remove this assumption by projecting the policy parameters $\theta_i(t)$ onto an appropriately chosen compact set after each update \eqref{eq:update_theta}. However, this approach raises the question of how to select this projection set so that it contains at least one local maximum of the problem~\ref{eq:approx_dec_rl}, as per Assumption~\ref{assum:stability_theta}, and also complicates the analysis of off-policy methods, as we discuss after we present Theorem~\ref{thm:convergence}.

Let $\chi_i(\theta)$ denote a function that maps the policy parameter $\theta$ to the optimal value function parameter by means of the policy evaluation algorithm \cite{maei2010toward}. Then, we can show the following result for the  function $\chi_i(\theta)$.
\begin{lemma}
	\label{lem:lipschitz_chi}
	Let policy parameter $\theta_i$ at agent $i$ be fixed, and let agent $i$ run the gradient TD (GTD) learning algorithm in \cite{maei2010toward} to evaluate this policy. Then, the local value function parameter $w_i$ converges to $\chi_i(\theta)$ almost surely (a.s.). Moreover, the function $\chi_i(\theta)$ is Lipschitz continous.
\end{lemma}
\begin{proof}
By the two-time scale nature of Algorithm~\ref{alg:dac} and the fact that each policy evaluation is performed independently by each agent $i$, the GTD algorithm run at agent $i$ behaves as a centralized algorithm. Therefore, the results in Lemma~\ref{lem:lipschitz_chi} can be directly shown using Lemma 4 and 5 in \cite{degris2012off}.
\end{proof}

Next, we have the following result on the stability of the value function parameter $w_i$.
\begin{lemma}
	\label{lem:stability_w}
	Given Assumption~\ref{assum:bounded_reward} and \ref{assum:stability_theta}, the value function parameter $w_i$ is  a.s. uniformly bounded over time.
\end{lemma}
\begin{proof}
	According to Lemma~\ref{lem:lipschitz_chi}, we have that $w_i$ converges to $\chi_i(\theta)$ a.s. Then, the boundness of $w_i$ is given by combining Assumption~\ref{assum:stability_theta} and the fact that the function $\chi_i(\theta)$ is $L_\chi-$Lipschitz continuous.
\end{proof}

In what follows, we stack all policy function parameters in a vector $\theta(t) = [\theta_i(t)^T, \dots, \theta_N(t)^T]^T$. The update $\theta(t)$ \eqref{eq:update_theta} can be compactly written as
\begin{equation}
	\label{eq:update_theta_1}
	\begin{split}
	\theta(t+1) = (W_t \otimes I)(\theta(t) + \alpha_\theta(t) \hat{\nabla} J(\theta(t))),
	\end{split}
\end{equation}
where $\hat{\nabla}J(\theta(t)) = \begin{bmatrix}
\vdots \\ \phi_{\pi}(s_{t+1}) ( \phi_{\pi}(s_{t+1})^T \chi_i(\theta_i(t))) \\ \vdots
\end{bmatrix}$ and $I$ is an identity matrix of the same dimension as $\theta_i(t)$. 
The expression of $\hat{\nabla}J(\theta(t))$ is due to Assumption~\ref{assum:function_compatibility}.
Moreover, define the disagreement between local policy parameters as $\theta_{\perp}(t) = \theta(t) - \mathbf{1} \otimes \bar{\theta}(t)$, where $\mathbf{1}$ is a vector of dimension $N$ and its entries are all $1$ and $\bar{\theta}(t) = \frac{1}{N} \sum_i \theta_i(t)$. We have the following lemma.
\begin{lemma}
	\label{lem:consensus}
	Given Assumptions~\ref{assum:stepsize}, \ref{assum:weight_matrix}, \ref{assum:stability_phi} and \ref{assum:stability_theta}, we have that $\sum_t \mathbb{E}[\|\theta_\perp(t)\|^2] < \infty$. Therefore, $\theta_{\perp}(t) \rightarrow 0$ a.s.
\end{lemma}
\begin{proof}
	First, we establish the dynamcis of $\theta_\perp(t)$. To achieve this, we introduce an operator $J_\perp := (I - \frac{1}{N}\mathbf{1}\mathbf{1}^T) \otimes I$. Then, multiplying both sides of \eqref{eq:update_theta_1} with $J_\perp$, and replacing $\theta(t)$ with $ \mathbf{1} \otimes \bar{\theta}(t) + \theta_{\perp}(t)$, we have that
	\begin{equation*}
	\begin{split}
	\theta_{\perp}(t) = & J_\perp(W_t \otimes I) \\
	& (\mathbf{1} \otimes \bar{\theta}(t-1) + \theta_{\perp}(t-1) + \alpha_\theta(t) \hat{\nabla} J(\theta(t-1))).
	\end{split}
	\end{equation*}
	Using Assumption~\ref{assum:weight_matrix}(a), we have that $J_\perp (W_t \otimes I)(\mathbf{1} \otimes \bar{\theta}(t)) = 0$ for all $t$. Therefore, we obtain
	\begin{equation*}
		\theta_{\perp}(t) =  J_\perp(W_t \otimes I) (\theta_{\perp}(t-1) + \alpha_\theta(t) \hat{\nabla}J(\theta(t-1))).
	\end{equation*}
	Taking the square of the Euclidean norm on both sides of the above equation, we have that
	\begin{equation*}
		\|\theta_\perp(t)\|^2 = \|\theta_{\perp}(t-1) + \alpha_\theta(t) \hat{\nabla}J(\theta(t-1))\|^2_{(W_t^T(I - \frac{1}{N}\mathbf{1}\mathbf{1}^T)W_t)},
	\end{equation*}
	where $\|v\|^2_M := v^TMv$ for any vector $v$ and matrix $M$. According to Assumption~\ref{assum:weight_matrix}(b,c), taking expectation over the random matrix $W_t$, given the filtration $\mathcal{F}_{t-1}$ and the random sample $(s(t), r(t))$, we have that
	\begin{equation*}
		\begin{split}
		& \mathbb{E}[\|\theta_\perp(t)\|^2 | \mathcal{F}_{t-1}, s(t), r(t)] \leq \\
		& \quad \quad \quad  \rho_W \|\theta_{\perp}(t-1) + \alpha_\theta(t) \hat{\nabla}J(\theta(t-1))\|^2 \\
		& \leq \rho_W(\|\theta_\perp(t-1)\|^2 + 2\alpha_\theta(t)\|\theta_\perp(t-1)\|\|\hat{\nabla}J(\theta(t-1)\| \\
		& \quad \quad \quad + \alpha_\theta(t)^2 \|\hat{\nabla}J(\theta(t-1)\|^2 ),
		\end{split}
	\end{equation*}
	where the second inequality above is by expanding the two norm and using the Cauchy-Swartz inequality. Taking the expectation of both sides of above inequality and using Jensen's inequality, we have that
	\begin{equation}
		\label{eq:ineq_1}
		\begin{split}
		 &\mathbb{E}[\|\theta_\perp(t)\|^2] \leq \rho_W \mathbb{E}[\|\theta_\perp(t-1)\|^2] \\ 
		 & \quad \quad + 2\rho_W\alpha_\theta(t) \sqrt{\mathbb{E}[\|\theta_\perp(t-1)\|^2\|\hat{\nabla}J(\theta(t-1)\|^2]} \\
		 & \quad \quad + \rho_W \alpha_\theta(t)^2 \mathbb{E}[\hat{\nabla}J(\theta(t-1)\|^2].
		\end{split}
	\end{equation}
	Recalling the expression for $\hat{\nabla}J(\theta(t-1)$ under \eqref{eq:update_theta_1}, and using  Assumption~\ref{assum:stability_phi} and Lemma~\ref{lem:stability_w}, we have that  $\|\hat{\nabla}J(\theta(t-1)\|$ can be bounded by a constant $K_1$ for all $t$. Denote $v(t) = \mathbb{E}[\|\theta_\perp(t)\|^2]$. Then, recalling that $\rho_W < 1$ due to Assumption~\ref{assum:weight_matrix}(b), \eqref{eq:ineq_1} can be written as
	\begin{equation*}
		v(t) \leq \rho_W v(t-1) + 2 \alpha_\theta(t) K_1 v(t-1) + \alpha_\theta(t)^2 K_1^2.
	\end{equation*}
	The above inequality is the same as (17) in \cite{bianchi2013convergence}. Since Assumptions~\ref{assum:weight_matrix} and \ref{assum:stepsize}imply that Assumptions 1 and 2 in \cite{bianchi2013convergence} are also satisfied, we can use the same proof as in Lemma 1 in \cite{bianchi2013convergence} to show that $\mathbb{E}[\|\theta_\perp(t)\|^2]$ satisfies that $\mathbb{E}[\|\theta_\perp(t)\|^2] < \infty$ and $\theta_{\perp}(t) \rightarrow 0$ a.s..
\end{proof}
Since $\theta_{\perp}(t) \rightarrow 0$ a.s., it is sufficient to study the dynamics of $\bar{\theta}(t)$. In what followss, we show that with the policy update in \eqref{eq:update_theta_1}, $\bar{\theta}(t)$ asymptotically approaches the following ODE dynamics
\begin{equation}
\label{eq:ode}
\dot{\bar{\theta}} = F(\bar{\theta}),
\end{equation}
where 
\begin{equation}
\label{eq:dynamic_ode}
F(\bar{\theta}) = \mathbb{E}[ \frac{1}{N} \phi_{\pi}(s)\phi_{\pi}(s)^T \sum_i \chi_i(\bar{\theta})].
\end{equation}

Note that it is standard to study the discrete-time dynamics \eqref{eq:update_theta} by relating them to the behavior of the ODE in \eqref{eq:ode}; see, e.g., relevant literature on RL \cite{degris2012off,zhang18continuous,zhang2018fullyICML} and stochastic optimization \cite{bianchi2013convergence}, as well as Chapter 2 in \cite{borkar2009stochastic} that establishes conditions on the step sizes and noise terms in the discrete-time dynamics so that they asymptotically approach their continuous-time ODE counterpart.

To show that the discrete-time dynamics of $\bar{\theta}(t)$ asymptotically approach the ODE \eqref{eq:ode}, we first multiply both sides of \eqref{eq:update_theta_1} with $\frac{1}{N}\mathbf{1}^T \otimes I$ on the left to obtain
\begin{flalign}
\label{eq:theta_bar_1}
& \bar{\theta}(t+1) = \frac{1}{N}(\mathbf{1}^T \otimes I) (\theta(t) + \alpha_\theta(t+1) \hat{\nabla} J(\theta(t))) & \nonumber \\
& = \bar{\theta}(t) + \alpha_\theta(t+1) \frac{1}{N} (\phi_{\pi}(s_{t+1})\phi_{\pi}(s_{t+1})^T) \sum_i \chi_i(\theta_i(t)) & \nonumber
\end{flalign}
Since $\sum_i \chi_i(\theta_i(t)) = \sum_i (\chi_i(\theta_i(t)) - \chi_i(\bar{\theta}_i(t)) + \chi_i(\bar{\theta}_i(t)))$, the above update of $\bar{\theta}(t)$ can be written in the following form
\begin{equation}
\label{eq:theta_bar_2}
\begin{split}
\bar{\theta}(t+1) =  \bar{\theta}(t) + & \alpha_\theta(t+1) F(\bar{\theta}(t)) \\
& + \alpha_\theta(t+1) \xi(t) + \alpha_\theta(t+1) r(t),
\end{split}
\end{equation}
where we have
\begin{subequations}
	\label{eq:conditions}
\begin{equation}
\label{eq:condition_1}
F(\bar{\theta}_t) = \mathbb{E}[ \frac{1}{N} \phi_{\pi}(s_{t+1})\phi_{\pi}(s_{t+1})^T \sum_i \chi_i(\bar{\theta}(t))],
\end{equation}
\begin{equation}
\label{eq:condition_2}
\begin{split}
\xi(t) = & \frac{1}{N} \phi_{\pi}(s_{t+1})\phi_{\pi}(s_{t+1})^T \sum_i \chi_i(\bar{\theta}(t)) \\
& - \mathbb{E}[ \frac{1}{N} \phi_{\pi}(s_{t+1})\phi_{\pi}(s_{t+1})^T \sum_i \chi_i(\bar{\theta}(t))].
\end{split}
\end{equation}
\begin{equation}
\label{eq:condition_3}
r(t) = \frac{1}{N} \phi_{\pi}(s_{t+1})\phi_{\pi}(s_{t+1})^T \sum_i (\chi_i(\theta_i(t)) - \chi_i(\bar{\theta}(t))).
\end{equation}
\end{subequations}
Then, to show that the discrete-time trajectory in \eqref{eq:theta_bar_2} approaches the continuous trajectory of \eqref{eq:ode}, we need to define the following functions generated by these trajectories; cf. Chapter 2.1 in \cite{borkar2009stochastic}. First, let $\bar{x}(n)$ denote a continuous piecewise linear function that passes through the discrete-time updates in \eqref{eq:theta_bar_2}, so that $\bar{x}(n(t))=\bar{\theta}(t)$ for $t\geq 0$ and $\bar{x}(n)= \bar{x}(n(t)) + \frac{\bar{x}(n(t+1)) - \bar{x}(n(t))}{n(t+1) - n(t)} (n - n(t))$ for $n(t)<n<n(t+1)$, where $n(0)=0$, $n(t)=\sum_{m=0}^{t-1}\alpha_{\theta}(m)$ and $n$ denotes the continuous time index. Moreover, define the function $x^s(n)$ that is the unique solution of the dynamical equation \eqref{eq:ode} for $n \geq s$ with initial condition $x^s(s) = \bar{\theta}(s)$, and the function $x_s(n)$ that is the unique solution of \eqref{eq:ode} for $n \leq s$ with the ending condition $x_s(s) = \bar{\theta}(s)$. 
Then, we can show the following result.

\begin{lemma}
	\label{lem:ode}
	Given Assumption~\ref{assum:stepsize}, \ref{assum:weight_matrix}, \ref{assum:stability_phi} and \ref{assum:stability_theta}, we have that for any $T > 0$, 
	\begin{equation*}
		\begin{split}
			& \lim_{s \rightarrow \infty} \sup_{n \in [s, s+T]} \|\bar{x}(n) - x^s(n)\| = 0, \; \text{a.s.} \\
			& \lim_{s \rightarrow \infty} \sup_{n \in [s, s-T]} \|\bar{x}(n) - x_s(n)\| = 0, \; \text{a.s.}.
		\end{split}
	\end{equation*}
\end{lemma}

\begin{proof}
According to Lemma 1 in Chapter 2 \cite{borkar2009stochastic}, it is sufficient to show that the following conditions are satisfied: (i) the function $F(\bar{\theta}_t)$ in \eqref{eq:condition_1} is Lipschtiz continuous, (ii) $\xi(t)$ in \eqref{eq:condition_2} satisfies that $\mathbb{E}[\xi(t)|\mathcal{F}_{t-1}] = 0$ and $\mathbb{E}[\|\xi(t)\|^2|\mathcal{F}_{t-1}] \leq K_2(1 + \|\bar{\theta}_t\|^2)$ a.s. for some constant $K_2 > 0$, and (iii) that  $\|r(t)\| \rightarrow 0$ a.s..  Note that Lemma 1 in \cite{borkar2009stochastic} requires conditions (i-ii), but only considers the dynamical equation \eqref{eq:theta_bar_2} without the noise term $r(t)$. The condition for the dynamical equation with noise $r(t)$ to approach its ODE counterpart is given by the third extension of Lemma 1 in Chapter 2.2 \cite{borkar2009stochastic}. And this extension requires condition (iii).

Combining Assumption~\ref{assum:stability_phi} and Lemma~\ref{lem:lipschitz_chi}, and recalling the definition in \eqref{eq:condition_1}, condition (i) is satisfied. In addition, from the construction of $\xi(t)$, it is simple to see that $\mathbb{E}[\xi(t)|\mathcal{F}_{t-1}] = 0$. Particularly, since $\phi_\pi(s)$ is bounded for all $s$ and $\chi_i(\bar{\theta}(t))$ is also bounded according to Assumption~\ref{assum:stability_phi} and Lemma~\ref{lem:lipschitz_chi}, we have that $\mathbb{E}[\|\xi(t)\|^2|\mathcal{F}_{t-1}]$ is uniformly bounded for all $t$. Therefore, the constant $K_2$ in condition (ii) must exist and this condition is satisfied. Finally, by the Lipschitz property of the function $\chi_i(\theta)$ shown in Lemma~\ref{lem:lipschitz_chi}, we have that
\begin{equation*}
	\begin{split}
	\|r(t)\| & \leq \frac{L_\chi}{N}\|\phi_{\pi}(s_{t+1})\phi_{\pi}(s_{t+1})^T\| 
	\sum_i \|\theta_i(t) - \bar{\theta}(t)\| \\
	& \leq \frac{L_\chi}{\sqrt{N}}\|\phi_{\pi}(s_{t+1})\phi_{\pi}(s_{t+1})^T\| 
	\|\theta_\perp(t)\|.
	\end{split}
\end{equation*}
Due to boundness of $\phi_\pi(s)$ and Lemma~\ref{lem:consensus}, we have that $\|r(t)\| \rightarrow 0$ a.s.. Therefore, condition (iii) is also satisfied.By conditions (i-iii), Assumption~\ref{assum:stepsize} and applying Lemma 1 in \cite{borkar2009stochastic}, the proof is complete.
\end{proof}
Before we state our main result, define the set $\Lambda = \{\bar{\theta}: (\mathbf{1}^T \otimes I) \hat{\nabla}J(\mathbf{1} \otimes \bar{\theta}) = 0 \}$ and make the following assumption.
\begin{assumption}
	\label{assum:invariant_set}
	We assume that set $\Lambda$ is compact. Meanwhile, the set $\sum_{i = 1}^N J_{i,\beta}(\Lambda)$ has an empty interior.
\end{assumption} 
This assumption is satisfied when the objective function $J_\beta(\theta)$ is smooth, according to Sard's theorem. It is a common assumption in the stochastic approximation and optimization literature, e.g., \cite{benaim2005stochastic,bianchi2013convergence}.
\begin{theorem}
	\label{thm:convergence}
	Given Assumptions \ref{assum:function_compatibility} and from \ref{assum:aperiordic} to \ref{assum:invariant_set}, $\theta_i(t)$ converges to the set $\Lambda$ a.s. for all $i$.
\end{theorem}
\begin{proof}
	Using Lemma~\ref{lem:consensus}, we need to show that $\bar{\theta}(t)$ given by \eqref{eq:theta_bar_2} converges to the set $\Lambda$. Moreover, using Lemma~\ref{lem:ode}, we need to show that the dynamics \eqref{eq:ode} converge to the set $\Lambda$. To achieve this, define function $f(\bar{\theta}) = -J_\beta(\bar{\theta}) = -\sum_{i=1}^N J_{i,\beta}(\bar{\theta})$, where $J_\beta(\theta)$ is the objective function of the central problem in \eqref{eq:approx_objective}. We shall prove that the function $f(\bar{\theta})$ can serve as a Lyapunov function to show stability of the set $\Lambda$ under dynamics \eqref{eq:ode}. 
	For this, we need to show that $\dot{f}(\bar{\theta}(n)) \leq 0$ for any solution $\bar{\theta}(n)$ of the ODE in \eqref{eq:ode}, where $n$ is the continuous time index, and that the inequality is strict for any $\bar{\theta} \notin \Lambda$.
	
	If the function $F(\bar{\theta})$ is the gradient of the function $f(\bar{\theta})$, then we can directly use Proposition 4 in \cite{bianchi2013convergence} to get the desired result. However, due to the proposed off-policy framework as we discussed in Section~\ref{sec:pf}, $F(\bar{\theta}_t)$ is only an approximation of the exact gradient. In what follows, we show that $F(\bar{\theta})$ behaves in a similar way as the exact gradient.
	Recalling the definition of $J_{\beta}(\bar{\theta}) = \mathbb{E}_{\rho^\beta}[Q^\pi(s, \pi_\theta(s))] = \mathbb{E}_{\rho^\beta}[\sum_iQ_i^\pi(s, \pi_\theta(s))]$, we have that
	\begin{equation*}
		\hat{\nabla}J_\beta(\bar{\theta}) = \mathbb{E}[\nabla_\theta \pi(s)|_{\theta = \bar{\theta}} \sum_i \nabla_a Q_i^\pi(s, a)|_{a = \pi_\theta(s)}].
	\end{equation*}
	Since the local value functions are parameterized by $Q_i^\pi(s,a) = (a - \pi(s))^T \nabla_\theta \pi(s)^T \chi_i(\bar{\theta})$, we have that $\nabla_a Q_i^\pi(s,a) = \nabla_\theta \pi(s)^T \chi_i(\bar{\theta})$ for all $i$. Plugging this into the above equation, we get
	\begin{equation}
		\hat{\nabla}J_\beta(\bar{\theta}) = \mathbb{E}[(\nabla_\theta \pi(s)\nabla_\theta \pi(s)^T)|_{\theta = \bar{\theta}} \sum_i \chi_i(\bar{\theta})].
	\end{equation}
	Furthermore, using the parameterization $\pi(s) = \phi(s)^T\theta$, we can obtain that the function $F(\bar{\theta})$ in \eqref{eq:dynamic_ode} is the off-policy gradient $\hat{\nabla}J_\beta(\bar{\theta})$. According to the policy improvement theorem (Theorem 1) in \cite{degris2012off}, we have that there exists $\epsilon > 0$, such that for all positive $\alpha_\theta < \epsilon$ and $\bar{\theta}' = \bar{\theta} + \alpha_\theta \hat{\nabla} J_\beta(\bar{\theta})$, we have that $J_\beta(\bar{\theta}') \geq J_\beta(\bar{\theta})$. And for all $\bar{\theta} \notin \Lambda$, the above inequality is strict. Considering the first order Taylor expansion of the value function $J_\beta(\bar{\theta}') = J_\beta(\bar{\theta}) + \langle \nabla J_\beta(\bar{\theta}), \bar{\theta}' - \bar{\theta} \rangle + o(\|\bar{\theta}' - \bar{\theta}\|)$, when $\alpha_\theta$ goes to $0$, the term $\langle \nabla J_\beta(\bar{\theta}), \bar{\theta}' - \bar{\theta} \rangle$ dominates $o(\|\bar{\theta}' - \bar{\theta}\|)$. Since $\langle \nabla J_\beta(\bar{\theta}), \bar{\theta}' - \bar{\theta} \rangle = \alpha_\theta \langle \nabla J_\beta(\bar{\theta}), \hat{\nabla} J_\beta(\bar{\theta}) \rangle$, combining with the policy improvement theorem, we have that $\langle \nabla J_\beta(\bar{\theta}), \hat{\nabla} J_\beta(\bar{\theta}) \rangle \geq 0$ for all $\bar{\theta}$. And if $\bar{\theta} \notin \Lambda$, we have that $\langle \nabla J_\beta(\bar{\theta}), \hat{\nabla} J_\beta(\bar{\theta}) \rangle > 0$. 
	
	Since $f(\bar{\theta}) = -J_\beta(\bar{\theta})$, we have that $\dot{f}(\bar{\theta}) = \langle \nabla f(\bar{\theta}), \dot{\bar{\theta}} \rangle = \langle -\nabla J_\beta(\bar{\theta}), \dot{\bar{\theta}} \rangle$. Since $\dot{\bar{\theta}} = F(\bar{\theta}) = \hat{\nabla}_\beta(\bar{\theta})$, we have that $\dot{f}(\bar{\theta}) =  -\langle \nabla J_\beta(\bar{\theta}), \hat{\nabla}J_\beta(\bar{\theta}) \rangle$. We conclude that $f(\bar{\theta})$ is a valid Lyapunov function that can be used to show stability of the set $\Lambda$ under dynamics \eqref{eq:ode}. Finally, given Assumption~\ref{assum:invariant_set}, combining Lemma~\ref{lem:ode} and Theorem 2 in \cite{bianchi2013convergence}, we have that $\bar{\theta}(t)$ converges to the set $\Lambda$ a.s.. Recalling that $\|\theta_\perp\| \rightarrow 0$ a.s., we obtain the desired result. The proof is complete.
\end{proof}

To conclude, we make the following remark on the two-time scale analysis we have employed in this paper. Specifically, we have considered the two-time scale update
\begin{equation*}
	\begin{split}
		& w_i(t+1) = w_i(t) + \alpha_w h(w_i(t), \theta_i(t)) + M_i(t+1), \\
	& \bar{\theta}(t+1) =  \bar{\theta}(t) + \alpha_\theta(t+1) F(\bar{\theta}(t)) \\
	& \quad \quad \quad \quad + \alpha_\theta(t+1) \xi(t) + \alpha_\theta(t+1) r(t).
	\end{split}
\end{equation*}
To analyze the above concurrent update in the two-time scale framework, according to \cite{borkar2009stochastic}, the noise $r(t)$ in the update of $\bar{\theta}(t)$ need to be a Martingale difference sequence with bounded momentum. However, as we have shown in Lemma~\ref{lem:ode}, the sequence $\{r(t)\}$ is related to the disagreement error and is not necessarily a Martingale difference sequence. Therefore, the analysis of the concurrent update scheme in \cite{borkar2009stochastic} cannot be directly applied here. However, as mentioned in the end of Chapter 6.1 \cite{borkar2009stochastic}, the two-time scale effect can also be achieved by the subsampling scheme. Specifically, let the local critic update $w_i(t)$ in the fast time scale and let actor keep its local policy estimate fixed and only update when the local critic updates $N_t$ steps. Then, according to Lemma~\ref{lem:lipschitz_chi}, when $N_t$ is chosen large enough, the local critic converges. Therefore, the convergence of the local actors can be analyzed in the two-time scale fashion as we have done in this section.

\section{Numerical Simulation}
\label{sec:simulation}

In this section, we illustrate our proposed algorithm and theoretical analysis using a distributed resource allocation example.
Specifically, we consider $6$ resource dispatch centers in an area of interest. These centers make decisions
as to how to allocate available resources amongst each other. For example, the resources can be vehicles or taxis that service passengers in a big city and the dispatch centers can control the number of vehicles in different neighborhoods in the city. Due to varying passenger demand, the dispatch centers may need to transfer vehicles from one another. We define the state of each center $i$ at time $t$ by $s_i(t)$ that captures the number of available resources $m_i(t)$ and the local demands $d_i(t)$. We also define the local action as $a_i(t) = \{a_{ij}(t)\}_{j \in \mathcal{N}_i}$, which denotes the amount of resources sent from center $i$ to its neighbor center $j$ at time $t$. In this simulation, we assume the 6 centers are located in a $2 \times 3$ grid. Each center communicates with its $1-$hop neighboring centers located at its up/down/left/right direction. We assume the demand at each agent is of sinusoidal shape with random noise, i.e.,
\begin{equation*}
	d_i(t) = A_i \sin(\omega_i t + \phi_i) + w_i(t),
\end{equation*}
where $\{A_i, \omega_i, \phi_i\}$ are randomly generated for all the agents. We denote by $T_i =  \frac{2\pi}{\omega_i}$ the periodicity of demand at agent $i$. The noise $w_i(t)$ is subject to a zero-mean Gaussian distribution $\mathcal{N}(0, \sigma_i^2)$, and we set $\sigma_i$ to be $10\%$ of $A_i$. Given this demand model, we let $s_i(t) = [m_i(t), \bar{t}_i(t)]^T$, where $\bar{t}_i(t)$ denotes the phase of the local demand. Moreover, we define the state transition function $T(s, a, w_i)$ as
\begin{equation*}
	m_i(t+1) = \Pi_{\mathcal{M}}[m_i(t) + \sum_{j \in \mathcal{N}_i} a_{ji} - \sum_{j \in \mathcal{N}_i} a_{ij} - d_i(t)]
\end{equation*}
\begin{equation*}
	\bar{t}_i(t+1) = 
	\begin{cases}
	& \bar{t}_i(t) + \delta \text{, if } \bar{t}_i(t) + \delta t < \frac{T_i}{2}, \\
	& \bar{t}_i(t) + \delta t - T_i \text{, if } \bar{t}_i(t) + \delta t > \frac{T_i}{2}.
	\end{cases}
\end{equation*}
where $\mathcal{M}$ is a compact set, $\Pi_{\mathcal{M}}$ is the projection operator onto the set $\mathcal{M}$, and $\delta t$ is the sampling interval. The local reward function is designed as
\begin{equation}
	 r_i(s_i(t)) =  
	\begin{cases}
	& 0   \quad \quad \quad \quad \;\; \;\text{ if } m_i(t) > 0, \\
	& -(-m_i(t))^3  \;\text{ if } m_i(t) < 0.
	\end{cases}
\end{equation}
This reward function penalizes agents for having negative resources but also does not reward them for accumulating too many resources.
Since the agents can only observe $s(t)$ and $a(t)$ but do not know the models for the demands $d_i(t)$, the transition function $T(s, a)$ or the reward functions $r_i(t)$, this problem becomes a distributed RL problem.

We apply Algorithm~\ref{alg:dac} to solve this problem. Specifically, we parmeterize the global policy function $\pi(s)$ using radial Gaussian basis functions (GBF) as $\pi(s) = \sum_{k = 1}^{n_\theta} \phi_k(s) \theta_k$, where
\begin{equation*}
	\phi_k(s) = \frac{1}{\sqrt{2\pi\sigma_k^2}} \exp(-\frac{\|s - c_k\|^2}{2\sigma_k^2}).
\end{equation*}
The feature parameters $\{c_k\}$ are randomly generated and $\{\sigma_k\}$ are adjusted by trial and error. All the agents in the network share the same basis functions for their policies. According to \ref{assum:function_compatibility}, the basis functions for the value functions $Q_i(s,a)$ are $\phi_{w_i}(s, a) = \phi_k(s)(a - \pi_{\theta_i}(s))$. As discussed at the end of Section~\ref{sec:convergence}, we utilize the subsampling to achieve a two-time scale effect. Specifically, we let the local actors update once every $20$ updates of the critics. The step sizes are chosen as $\alpha_w = \alpha_v = \alpha_u = t^{-0.55}$ and $\alpha_\theta = t^{-0.65}$.

\begin{figure}
	\centering
	\includegraphics[width = .9\columnwidth]{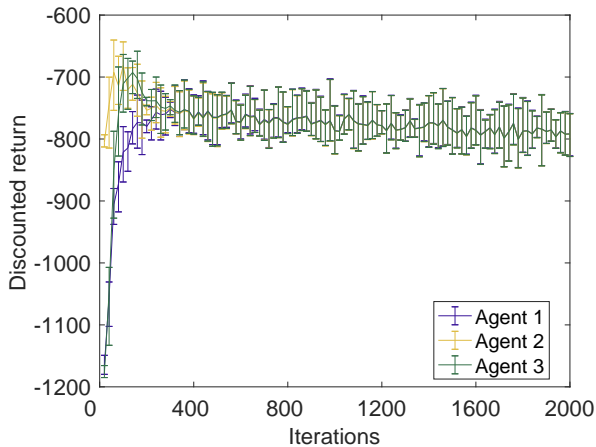}
	\caption{Performance of Algorithm~\ref{alg:dac} on a distributed resource allocation problem. Each curve represents the policy improvement of one agent over iterations. The algorithm is run 5 times. The curves represent the mean performance over the 5 trials and the error bars represent the variance of the performance.}
	\label{fig:curve}
\end{figure}

We run Algorithm~\ref{alg:dac} for $5$ times with the same initialization. Specifically, all agents start with the same initial policy parameters and value function parameters at the beginning of each trial. The randomness in these trials is due to the noise in the demand model and the random exploration policy $\beta(s)$. In Figure~\ref{fig:curve}, we demonstrate that the accumulated return using each agent's policy estimate increases. Each point in the curve in Figure~\ref{fig:curve} is achieved in the following way: fix the policy parameter at agent $i$ at the current iteration, let all the agents execute this policy for $200$ steps and compute the aggregate accumulated reward over the whole network, run the aforementioned process for $20$ times and take the mean of the accumulated rewards. This mean value is used as the performance indicator for the current policy parameter at agent $i$ at the current iteration. Finally, observe in Figure~\ref{fig:curve} that by applying Algorithm~\ref{alg:dac}, the local agent's policies achieve consensus. The policy estimate of agents 1 and 3 improves, while the local of agent 2 degrades. The reason is that the distributed RL problem~\eqref{eq:approx_dec_rl} is nonlinear. Algorithm~\ref{alg:dac} is only guaranteed to converge to a local optimizer.  We observe that the perforamance of the policy estimate at each local agent degrades slightly as the iteration increase, but eventually converges due to the decreasing step size rule. This behavior is caused by the variance in the policy gradient estimate in the Actor Critic method, and can also be observed in other literature on Actor Critic methods, e.g., \cite{zhang18continuous}.
\section{CONCLUSIONS}
\label{sec:conclusion}
In this paper, we proposed a distributed actor critic algorithm to solve multi-agent reinforcement learning problems. Specifically, we assumed that every agent has a local estimate of the global optimal policy function, updates its local estimate with the local value function estimate, and we introduced an additional consenesus step on these local estimates so that the agents asymptotically achieve agreement on the global optimal policy. We anlayzed the convergence of the proposed algorithm and demonstrated its effectiveness on a distributed resource allocation example. Compared to existing distributed actor critic methods for RL, using policy consensus does not require the agents to share their local tasks with each other.

\addtolength{\textheight}{-12cm}   

\bibliographystyle{IEEEtran}
\bibliography{bib_Yan}

\end{document}